\documentclass[letterpaper]{article}
\usepackage[preprint]{aaai2027}
\usepackage[hyphens]{url}
\usepackage{graphicx}
\usepackage{natbib}
\usepackage{caption}
\usepackage{amsmath}
\usepackage{amssymb}
\usepackage{booktabs}
\usepackage{multirow}
\usepackage{bm}
\usepackage{algorithm}
\usepackage{algorithmic}

\newcommand{\fs}{f_s}
\newcommand{\fl}{f_l}

\newcommand{\Dcal}{\mathcal{D}_{\mathrm{cal}}}

\newcommand{\Dval}{\mathcal{D}_{\mathrm{val}}}
\newtheorem{proposition}{Proposition}

\title{Calibration-Aware Uncertainty Cascades for \\ Efficient Heterogeneous Model Collaboration}
\author{
    Yilin Zhang\equalcontrib,
    Han Jiang\equalcontrib,
    Cai Xu,
    Ying Liu,
    Wei Zhao
}
\affiliations{
    School of Computer Science and Technology, Xidian University\\
    Xi'an, China\\
    ylzhang\_3@stu.xidian.edu.cn, han522708@gmail.com, cxu@xidian.edu.cn,\\
    ying210281@163.com, ywzhao@mail.xidian.edu.cn
}

\begin{document}

\maketitle

\begin{abstract}


Heterogeneous model collaboration seeks to exploit the complementary strengths of different models to balance predictive performance and inference cost. Existing approaches typically rely either on trained routers, which tie routing decisions to a fixed task and model pool, or on raw-confidence cascades, whose thresholds lack consistent reliability semantics across heterogeneous models. Consequently, these approaches adapt poorly to changing model pools and deployment budgets.
We propose Calibration-Aware Uncertainty Cascades (CAUC), a simple post-hoc framework that independently calibrates each model's confidence and selects deployment policies using validation data. The resulting calibrated confidence scores establish a common reliability scale for accepting an early prediction, invoking a stronger model, or selectively combining model outputs. This unified decision criterion decouples deployment policies from any particular model pool or operating budget. We further show theoretically that calibration gives confidence thresholds an explicit selective-risk interpretation, whereas uncalibrated scores offer no comparable reliability guarantee.
Extensive experiments demonstrate that, across six language benchmarks, CAUC achieves an average relative accuracy improvement of 1.9\% over strong-model-only inference while avoiding approximately 47\% of strong-model calls. 
On image classification benchmarks, it maintains or improves predictive performance while reducing measured GFLOPs by up to 57\%.

\end{abstract}

\section{Introduction}


Modern AI deployments rely on heterogeneous models with different capabilities and computational costs. Compact models can handle many inputs efficiently on resource-constrained devices or low-cost servers, whereas larger models generally provide stronger predictive performance at substantially higher latency and computation. Always invoking the strongest model is therefore wasteful when less expensive models can already solve many inputs correctly. The central deployment question is how much computation each input should receive: whether to stop with the current model, invoke a stronger one, or combine complementary predictions. This trade-off arises in both edge--cloud vision systems that coordinate device and server inference and LLM services that select among models with different capabilities and prices \citep{kag2023efficient,TMLR2024frugalgpt}. The objective is consequently not accuracy alone, but predictive performance under a computation or service budget.

A common approach is to train a router that selects which model should process each query. Recent routers learn query--model compatibility from benchmark outcomes, preference data, or query and model representations, sometimes using contrastive objectives \citep{shnitzer2024routing,ong2025routellm,NIPS2024routerDC,ICLR2025embedllm}. Although these methods can exploit rich supervision, their decision rules are typically coupled to the training tasks, candidate model pool, and cost definition. Adding or replacing a model, or changing the deployment budget, may therefore require new performance labels and router retraining. This dependence limits their flexibility in model libraries whose members and operating costs evolve over time.

Uncertainty-based cascades provide a lightweight alternative: models are evaluated in increasing order of cost, and a confidence threshold determines whether the system stops or invokes a stronger model \citep{jitkrittum2023confidence,gupta2024cascades,ramirez2024optimising}. However, raw confidence is often an unreliable proxy for correctness. The same score can correspond to substantially different empirical accuracies across models and datasets, making the behavior of a shared threshold unpredictable. Alternatively, tuning a separate threshold for every model complicates model replacement and cascade expansion. Conventional cascades also discard earlier predictions after deferral, even when their errors are complementary to those of later models.

\begin{figure*}[t!]
\centering
\includegraphics[width=0.95\textwidth]{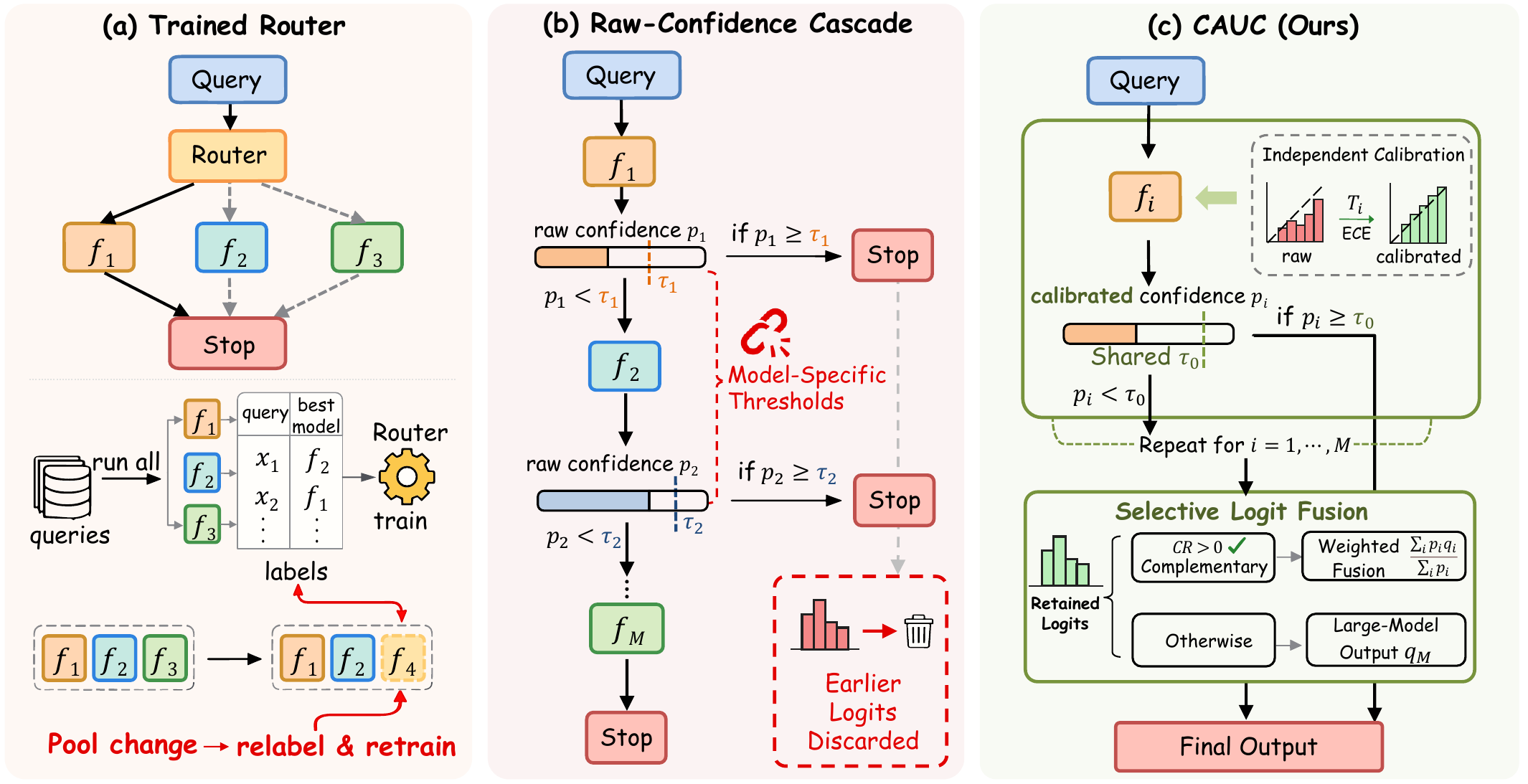}
\caption{
\textbf{Comparison of collaboration paradigms for heterogeneous models.}
(a) Trained routers need labeled routing data and retraining after model-pool changes.
(b) Raw-confidence cascades use model-specific thresholds and discard earlier predictions upon deferral.
(c) CAUC independently calibrates models to enable a shared threshold across the cascade. It selectively fuses retained logits when they are complementary; otherwise, it uses the final model's prediction.
}
\label{fig:motivation}
\end{figure*}


To address these limitations, we propose Calibration-Aware Uncertainty Cascades (CAUC), a post-hoc framework for efficient heterogeneous model collaboration. CAUC independently calibrates the confidence of each model, establishing a common reliability scale for accepting an early prediction or invoking a stronger model. For deferred inputs, CAUC selectively combines model outputs when their predictions are complementary. This separation between model-specific reliability estimation and deployment policy allows models to be replaced or added through lightweight post-hoc calibration and policy updates, without training a task-specific router or tuning a separate stopping threshold for every model. We further extend CAUC to multi-model cascades through recursive fusion, allowing predictions from earlier models to contribute to the final decision rather than being discarded after deferral.



Our contributions are summarized as follows:
\begin{itemize}
    \item We propose CAUC, a post-hoc framework that independently calibrates heterogeneous models and uses calibrated confidence as a common reliability scale for early acceptance, model escalation, and selective output fusion.

    \item By aligning confidence semantics across heterogeneous models, CAUC accommodates model replacement and cascade expansion through lightweight calibration and policy updates, rather than router retraining or joint retuning of model-specific thresholds.


    \item We provide theoretical justification for calibrated confidence thresholding, showing that calibration gives thresholds an explicit selective-risk interpretation and supports near-optimal decisions under a cost-sensitive objective.
    

    \item Extensive experiments on 6 LLM benchmarks and 3 image classification datasets demonstrate that CAUC achieves superior performance-cost trade-offs across diverse model combinations and deployment budgets.
\end{itemize}

\section{Related Work}

\paragraph{Model Routing.} 

Model routing selects one model from a candidate pool before inference. Existing approaches differ mainly in how they estimate query--model compatibility. Benchmark Routing learns correctness predictors from model outcomes on benchmark datasets, whereas RouteLLM derives pairwise routing decisions from preference data \citep{shnitzer2024routing,ong2025routellm}. Representation-based routers seek more transferable compatibility signals: RouterDC jointly learns query and model embeddings through dual contrastive objectives, while GraphRouter explicitly models relations among tasks, queries, and LLMs as a heterogeneous graph \citep{NIPS2024routerDC,ICLR2025graph}. EmbedLLM decouples reusable model representations from downstream routing decisions, whereas RadialRouter uses a lightweight structured encoder to capture query--model relations efficiently \citep{ICLR2025embedllm,jin2025radialrouter}. Collectively, these methods exploit specialization within a model pool, but depend on an auxiliary compatibility predictor learned for particular tasks and candidate models. CAUC addresses a different decision stage: rather than selecting a model before inference, it decides whether to stop after observing a prediction, using calibrated correctness estimates without training a separate query router.

\paragraph{Cascades and Adaptive Invocation.}
Model cascades invoke models in increasing order of cost and stop once an intermediate prediction is considered reliable. This principle appears in adaptive image inference through learned exits or transitions between classifiers, and in language-model systems that optimize service order and stopping policies under a budget \citep{bolukbasi2017adaptive,huang2018multiscale,TMLR2024frugalgpt,dekoninck2025unified}. The central design problem is the deferral criterion. Confidence-based deferral is effective only when the score is sufficiently related to correctness and downstream recoverability \citep{jitkrittum2023confidence}; recent LLM cascades therefore study token-level uncertainty, post-hoc representations, and generation margins as alternatives to a neural query router \citep{gupta2024cascades,ramirez2024optimising}. These approaches reduce strong-model calls, but their decision statistics and suitable thresholds can vary substantially across tasks and models. Conventional cascades also replace the early prediction after deferral, even when its errors are complementary to those of the stronger model. CAUC follows the cascade paradigm but calibrates the stopping statistic explicitly and retains selective fusion only as an action for deferred, ambiguous examples.



\paragraph{Confidence Calibration}
Uncertainty-based collaboration requires more than ranking easy and difficult inputs: a threshold should have a consistent reliability meaning. Post-hoc calibration aligns predictive confidence with empirical correctness without retraining the base model. Temperature scaling provides a strong baseline for image classifiers, while adaptive temperature scaling extends this idea to language-model option logits \citep{guo2017calibration,EMNLP2024calibrating}. Selective prediction complements calibration by measuring the risk among accepted examples as coverage changes \citep{geifman2017selective}; ensemble and fusion methods further show that complementary predictions can improve uncertainty when additional computation is available \citep{jiang2023llmblender,liu2025coolfusion,zhou2025asymmetric}. Prior cascade studies primarily use uncertainty as a ranking or deferral signal. CAUC instead treats calibrated correctness probability as a common interface across heterogeneous models, giving its stopping threshold an explicit selective-risk interpretation and allowing the policy to be retuned without jointly retraining the models or a router.

\section{Calibration-Aware Uncertainty Cascade}

\subsection{Problem Formulation}

We consider a prediction task with a finite candidate label set $\mathcal{Y}$, where $K=|\mathcal{Y}|$. Each input $x$ is paired with a target $y\in\mathcal{Y}$. In image classification, the candidates are visual classes; in multiple-choice question answering, they are the answer options provided for the question. Both settings therefore produce a score for every candidate and return the highest-scoring one. For simplicity, we explain the method using a two-model cascade consisting of a small model $\fs$ and a large model $\fl$. Model $m\in\{s,l\}$ produces logits $\bm{z}_m(x)\in\mathbb{R}^{K}$ and predicts $\hat y_m(x)=\arg\max_k \bm{z}_{m,k}(x)$.

The small model is always evaluated first. The system can accept its prediction or defer the input to the large model, after which it returns either the large-model prediction or a combination of both outputs. Let $c_s$ and $c_l$ denote the incremental costs of evaluating the small and large models. If $\rho$ is the fraction of inputs sent to the large model, the expected inference cost is
\begin{equation}
    \bar c=c_s+\rho c_l.
    \label{eq:pair-cost}
\end{equation}
Our goal is to reduce $\rho$ and hence $\bar c$ while maintaining or improving predictive performance relative to always using the large model.

\subsection{Calibrated Confidence for Cascading}

Due to differences in architecture, model scale, and data distribution, the same raw maximum probability may correspond to different empirical accuracies across heterogeneous models, making such probabilities not directly comparable.
CAUC addresses this problem by first calibrating each model so that confidence has a common correctness interpretation. It then uses calibrated confidence to decide whether the cascade should stop at the small model or invoke the large model. For deferred inputs, CAUC further checks whether the small-model output is complementary enough to retain. We first describe this two-model procedure and then introduce a recursive extension for longer model chains.

The cascade decision requires a score that means the same thing for both models. Native softmax probabilities often fail this requirement: one model can be overconfident while another is underconfident, so a shared raw-confidence threshold can accept samples with very different error rates. We therefore fit each model independently on a held-out calibration set $\Dcal=\{(x_i,y_i)\}_{i=1}^{N_{\mathrm{cal}}}$. Specifically, model $m$ receives one scalar temperature $T_m>0$, selected by minimizing its negative log likelihood on $\Dcal$:
\begin{equation}
    T_m=\arg\min_{T>0}-\frac{1}{N_{\mathrm{cal}}}
    \sum_{i=1}^{N_{\mathrm{cal}}}
    \log\!\left[\operatorname{softmax}\!\left(\bm{z}_m(x_i)/T\right)_{y_i}\right].
    \label{eq:temperature-fit}
\end{equation}
Both the image and language-model experiments use this standard scalar temperature scaling. For a new input $x$, the calibrated probability vector $\bm{\pi}_m$ and confidence $p_m$ are
\begin{equation}
\begin{aligned}
    \bm{\pi}_m(x)&=\operatorname{softmax}\!\left(\bm{z}_m(x)/T_m\right),\\
    p_m(x)&=\max_k\bm{\pi}_{m,k}(x).
\end{aligned}
    \label{eq:cal-score}
\end{equation}
Here, $p_m(x)$ estimates the probability that model $m$'s prediction on $x$ is correct. Since scaling the logits by a positive temperature preserves their ordering, calibration changes the confidence score without changing the predicted label.

CAUC uses the observed reliability of the large model as its stopping target. Let
\begin{equation}
    \widehat A_l=\frac{\left|\{i:\hat y_l(x_i)=y_i\}\right|}
    {N_{\mathrm{cal}}},
    \qquad \tau_0=\widehat A_l,
    \label{eq:single-threshold}
\end{equation}
where $\widehat A_l$ is the large model's empirical accuracy on $\Dcal$ and $\tau_0$ is the shared stopping threshold. CAUC accepts $\hat y_s(x)$ when $p_s(x)\geq\tau_0$; otherwise, it evaluates $\fl$. The rule has a direct interpretation: the small model answers only when its estimated probability of being correct reaches the reference accuracy observed for the large model. Since each confidence is calibrated to the same event, replacing either model requires fitting only its scalar temperature and recomputing the corresponding calibration statistics, rather than training an additional decision model.

\subsection{Selective Collaboration after Deferral}

Deferring an input does not imply that the small-model output is useless. Its prediction may correct errors made by the large model, but indiscriminate fusion can also replace correct large-model answers. CAUC decides whether to retain the small model by measuring this trade-off on $\Dcal$. Among calibration samples for which the small model is more confident than the large model, $p_s(x_i)>p_l(x_i)$, let $N_{\mathrm{gain}}$ count cases where the small model is correct and the large model is wrong, and let $N_{\mathrm{loss}}$ count the reverse cases. We define the calibration-set complementarity rate as
\begin{equation}
    {\mathrm{CR}}(s,l)=
    \frac{N_{\mathrm{gain}}-N_{\mathrm{loss}}}{N_{\mathrm{cal}}}.
    \label{eq:complementarity-rate}
\end{equation}
A positive value indicates that trusting the more confident small model produces more corrections than harmful replacements. CAUC therefore enables fusion only when ${\mathrm{CR}}(s,l)>0$; otherwise, every deferred input uses the large-model prediction directly. This sign test is performed once on the calibration set and introduces no learned decision network.

When fusion is enabled, we first place heterogeneous logits on comparable scales. Let $\bm{\ell}_m(x)=\bm{z}_m(x)/T_m$ be the calibrated logits. We estimate their global mean $\mu_m$ and standard deviation $\sigma_m$ on $\Dcal$ and form $\bm{q}_m(x)=(\bm{\ell}_m(x)-\mu_m)/(\sigma_m+\varepsilon)$, where $\varepsilon>0$ prevents division by zero. For a deferred sample, CAUC computes
\begin{equation}
\begin{aligned}
    \bm{z}_{\mathrm{fuse}}(x)
      &=\frac{p_s(x)\bm{q}_s(x)+p_l(x)\bm{q}_l(x)}{p_s(x)+p_l(x)},\\
    \hat y_{\mathrm{fuse}}(x)
      &=\arg\max_k z_{\mathrm{fuse},k}(x).
\end{aligned}
    \label{eq:cauc-fusion}
\end{equation}
Thus, $\bm{q}_s$ and $\bm{q}_l$ provide standardized class evidence, while $p_s$ and $p_l$ determine how strongly each model contributes on the current input. In summary, CAUC fits independent temperatures, sets the large-model reference threshold and complementarity sign on $\Dcal$, accepts reliable small-model predictions, and applies either direct fallback or selective fusion to the remaining inputs. These components form one calibration-driven cascade without a task-specific selector.

\subsection{Recursive Fusion Extension}

CAUC selectively fuses two model outputs after deferral. To preserve useful evidence across a longer model chain, we extend it to CAUC-RF, which recursively accumulates calibrated outputs while retaining the same stopping rule. Suppose $M$ models $f_1,\ldots,f_M$ are ordered from inexpensive to expensive, and let $\bm{\ell}_j=\bm{z}_j/T_j$ denote the individually calibrated logits of model $j$. CAUC-RF extends the endpoint fusion by maintaining a recursive score state. It starts with $\bm{r}_1=\bm{\ell}_1$ and, after evaluating model $j\geq2$, updates

\begin{equation}
    \bm{r}_j=\frac{\bm{r}_{j-1}/\alpha_j+\bm{\ell}_j/\beta_j}{2},
    \qquad \alpha_j,\beta_j>0.
    \label{eq:rf-state}
\end{equation}
The pair temperatures $\alpha_j$ and $\beta_j$ are fitted sequentially on $\Dcal$ by minimizing the NLL of $\operatorname{softmax}(\bm{r}_j)$. They control the relative scale of the accumulated evidence and the newly called model. The confidence after stage $j$ is
\begin{equation}
    p^{(j)}(x)=\max_k\operatorname{softmax}(\bm{r}_j(x))_k.
    \label{eq:rf-confidence}
\end{equation}

\begin{figure}[t]
\centering
\includegraphics[width=\columnwidth]{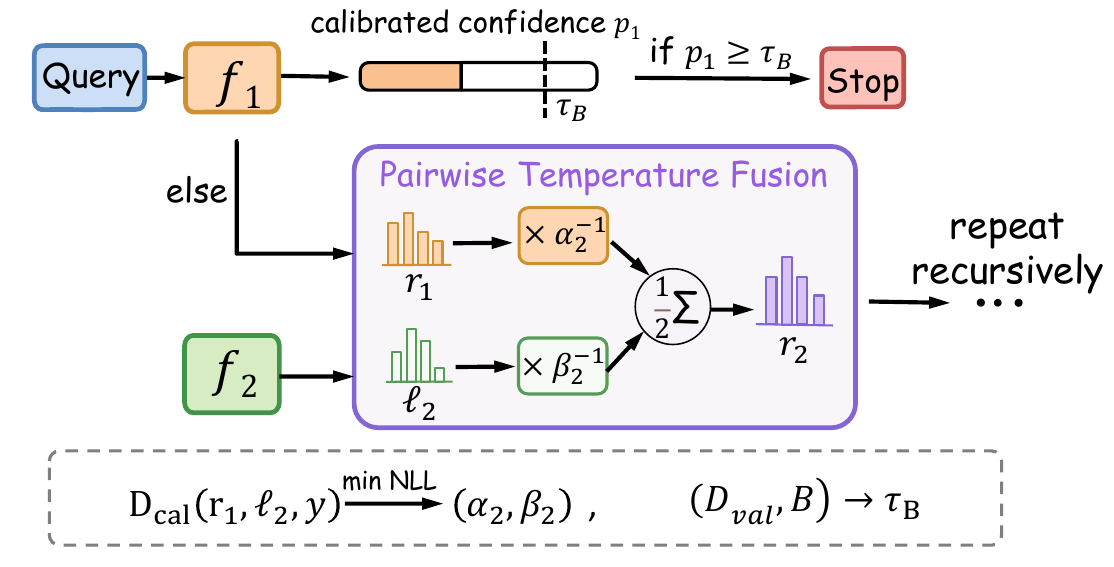}
\caption{CAUC-RF recursively fuses calibrated outputs using learned pairwise temperatures and applies a shared confidence threshold for early stopping.}
\label{fig:cauc-rf-method}
\end{figure}

Unlike the main CAUC rule, this extension can adapt its operating point to an explicit budget. On a calibration-independent cascade-validation set $\Dval$, it selects one threshold shared by all nonfinal stages:
\begin{equation}
    \tau_B\in\arg\max_{\tau:\,\widehat c_{\mathrm{val}}(\tau)\leq B}
    \widehat A_{\mathrm{val}}(\tau),
    \label{eq:rf-threshold}
\end{equation}
where $B$ is the deployment budget, $\widehat c_{\mathrm{val}}(\tau)$ is the measured cascade cost, and $\widehat A_{\mathrm{val}}(\tau)$ is the corresponding accuracy on $\Dval$. Inference stops at the first stage satisfying $p^{(j)}\geq\tau_B$; the final stage always returns its recursive prediction. We report this recursive, budget-tuned extension as \textbf{CAUC-RF}. It retains CAUC's calibrated-confidence interface and adds recursive evidence accumulation without introducing a learned selector.

\begin{table*}[!t]
\centering
\small
\setlength{\tabcolsep}{2.0pt}
\begin{tabular*}{\textwidth}{
@{\extracolsep{\fill}}
ll
rr rr rr rr rr rr
@{}
}
\toprule
\multirow{2}{*}{\textbf{Paradigm}} & \multirow{2}{*}{\textbf{Method}}
& \multicolumn{2}{c}{\textbf{MMLU}}
& \multicolumn{2}{c}{\textbf{LogiQA}}
& \multicolumn{2}{c}{\textbf{MathQA}}
& \multicolumn{2}{c}{\textbf{MedMCQA}}
& \multicolumn{2}{c}{\textbf{PIQA}}
& \multicolumn{2}{c}{\textbf{SocialIQA}} \\
\cmidrule(lr){3-4}\cmidrule(lr){5-6}
\cmidrule(lr){7-8}\cmidrule(lr){9-10}\cmidrule(lr){11-12}
\cmidrule(lr){13-14}
\multicolumn{2}{c}{} & Acc. & Cost & Acc. & Cost & Acc. & Cost
& Acc. & Cost & Acc. & Cost & Acc. & Cost \\
\midrule
\multirow{2}{*}{Single-Model} & Ministral-3-8B
& 75.94 & 8.00 & 57.45 & 8.00 & 44.44 & 8.00
& 64.06 & 8.00 & 86.34 & 8.00 & 73.61 & 8.00 \\
& Gemma-4-26B
& 79.08 & 26.00 & 58.22 & 26.00 & 45.89 & 26.00
& 65.48 & 26.00 & 89.61 & 26.00 & 76.44 & 26.00 \\
\midrule
\multirow{5}{*}{Router} & RouterDC
& 78.89 & 25.58 & 58.22 & 26.00 & 44.46 & 23.80
& 65.34 & 25.18 & 89.61 & 26.00 & 76.44 & 26.00 \\
& Benchmark Routing
& 78.34 & 16.28 & 58.06 & 14.47 & 45.33 & 15.55
& 64.81 & 16.60 & 87.81 & 35.35 & 75.22 & 13.60 \\
& RouteLLM
& 78.38 & 15.98 & 58.37 & 20.00 & 44.62 & 20.28
& 64.42 & 15.98 & 87.43 & 16.67 & 74.91 & 16.17 \\
& GraphRouter
& 76.22 & 9.66 & 56.68 & 11.20 & 43.65 & 8.04
& 64.45 & 13.50 & 86.67 & 9.23 & 73.61 & 8.89 \\
& EmbedLLM
& 78.09 & 18.78 & 56.68 & 19.40 & 44.62 & 15.97
& 65.48 & 21.80 & 89.06 & 19.71 & 75.54 & 19.57 \\
\midrule
\multirow{3}{*}{Cascade} & FrugalGPT
& 78.61 & 17.73 & 58.53 & 25.21 & 45.70 & 17.27
& 64.74 & 20.59 & 88.08 & 21.44 & 74.82 & 20.72 \\
& Margin Sampling
& 79.80 & 20.34 & 58.68 & 25.13 & 47.37 & 24.79
& 66.83 & 21.20 & 89.61 & 21.06 & 76.17 & 20.54 \\
& Post-Hoc-Embed
& 78.81 & 21.02 & 58.22 & 24.41 & 46.13 & 25.14
& 65.38 & 20.89 & 88.63 & 20.86 & 76.08 & 20.51 \\
\midrule
Ensemble & Asymmetric Duo
& \textbf{80.65} & 34.00 & \textbf{59.91} & 34.00 & \textbf{48.98} & 34.00
& \textbf{68.54} & 34.00 & \textbf{90.37} & 34.00 & \textbf{76.57} & 34.00 \\
\midrule
\multirow{2}{*}{Ours} 
& \textbf{CAUC}
& \underline{80.57} & 20.44 & 59.60 & 24.51 & \underline{47.98} & 24.20
& \underline{67.97} & 20.77 & 90.04 & 20.00 & 76.39 & 20.43 \\
& \textbf{CAUC-RF}
& 80.51 & 20.44 & \underline{59.91} & 24.51 & 47.70 & 24.20
& 67.90 & 20.77 & \underline{90.10} & 20.00 & \underline{76.53} & 20.43 \\
\bottomrule
\end{tabular*}
\caption{Two-model multiple-choice results for Ministral-3-8B $\rightarrow$ Gemma-4-26B. Accuracy is in percentage points; Cost is the average inference cost under the shared accounting. Bold indicates the highest overall accuracy, and underlining the highest accuracy among policies that can terminate without evaluating both models for every example.}
\label{tab:llm-main}
\end{table*}

\section{Theoretical Analysis}
We establish two theoretical analysis for calibrated thresholding in CAUC. The first relates the stopping threshold to accepted-prediction accuracy, while the second proves its near-optimality under a cost-sensitive objective.

\subsection{Why Calibration Matters}

Let $P\in[0,1]$ be a model's calibrated confidence on a randomly drawn input, and let $C$ equal one when its prediction is correct and zero otherwise. The model's actual accuracy at confidence $p$ is $\eta(p)=\Pr(C=1\mid P=p)$. We summarize its average calibration error by $\epsilon=\mathbb{E}|\eta(P)-P|$. For a stopping threshold $\tau$, let $q_\tau=\Pr(P\geq\tau)>0$ be the fraction of inputs accepted at this stage.

\begin{proposition}[Calibration controls accepted accuracy]
For every $\tau\in[0,1]$,
\begin{equation}
    \Pr(C=1\mid P\geq\tau)
    \geq \tau-\frac{\epsilon}{q_\tau}.
    \label{eq:cal-bound}
\end{equation}
\end{proposition}
\noindent\textit{Proof.}
The accepted accuracy is $\mathbb{E}[\eta(P)\mid P\geq\tau]$. The mean reported confidence in this group is at least $\tau$, while its mean calibration gap is at most $\epsilon/q_\tau$. Subtracting this gap gives Eq.~\ref{eq:cal-bound}. \hfill$\square$

This bound gives the stopping threshold an accuracy interpretation. Without calibration, a model could report confidence one on every input while having arbitrarily low accuracy, so the same threshold would provide no reliability guarantee. Calibration controls the accuracy of accepted samples; confidence discrimination determines how many samples can stop early.



\subsection{Near-Optimality of Unified-Threshold Cascading}

Consider the decision after the small model has been evaluated. Let $P_s$ be its calibrated confidence and let $\eta_s(p)$ be its true probability of being correct when it reports confidence $p$. Accepting the small model then has expected error $1-\eta_s(p)$. Deferring to the large model incurs
\begin{equation}
    L_{\mathrm{large}}=r_l+\lambda c_l,
\end{equation}
where $r_l$ is the large model's error rate, $c_l$ is its additional inference cost, and $\lambda\geq0$ specifies how strongly cost is penalized. We assume this fallback loss does not vary with $P_s$. For the nontrivial case $\tau^\star=1-L_{\mathrm{large}}\in[0,1]$, the better action is to accept the small model when $\eta_s(p)\geq\tau^\star$.

\begin{proposition}[Unified-threshold regret]
If the small model is perfectly calibrated, so that $\eta_s(p)=p$, accepting it when $P_s\geq\tau^\star$ is optimal. With average calibration error $\epsilon_s=\mathbb{E}|\eta_s(P_s)-P_s|$, the excess expected loss over an oracle that knows $\eta_s$ satisfies
\begin{equation}
    \mathcal R_{\mathrm{threshold}}-
    \mathcal R_{\mathrm{oracle}}\leq\epsilon_s.
    \label{eq:threshold-regret}
\end{equation}
\end{proposition}
\noindent\textit{Proof.}
The threshold rule and oracle differ only when $P_s$ and $\eta_s(P_s)$ fall on opposite sides of $\tau^\star$. Whenever this happens, the extra loss is no larger than $|\eta_s(P_s)-P_s|$. Averaging over inputs gives Eq.~\ref{eq:threshold-regret}; perfect calibration makes the gap zero. \hfill$\square$

The threshold depends on the fallback error and deployment cost, not on the small model architecture. It can therefore be reused after replacing the small model, provided that the new model is calibrated and the fallback setting remains unchanged. The CAUC default $\tau_0=\widehat A_l$ corresponds to matching the large model's error when cost is not explicitly penalized; CAUC-RF instead selects a budget-specific threshold on cascade-validation data.


\section{Experiments}

\subsection{Experimental Setup}

\begin{table*}[t]
\centering
\footnotesize
\setlength{\tabcolsep}{4pt}
\begin{tabular}{@{}ll*{4}{cc}@{}}
\toprule
{\multirow{2}{*}{\textbf{Paradigm}}} & {\multirow{2}{*}{\textbf{Method}}}
& \multicolumn{2}{c}{\textbf{Caltech256}}
& \multicolumn{2}{c}{\textbf{iWildCam}}
& \multicolumn{2}{c}{\textbf{iWildCam-OOD}}
& \multicolumn{2}{c}{\textbf{ImageNet}} \\
\cmidrule(lr){3-4}\cmidrule(lr){5-6}
\cmidrule(lr){7-8}\cmidrule(lr){9-10}
\multicolumn{2}{c}{} & Acc. & Rel. GFLOPs & F1 & Rel. GFLOPs
& F1 & Rel. GFLOPs & Acc. & Rel. GFLOPs \\
\midrule
\multirow{2}{*}{Single-Model} & MnasNet-0.75
& 84.84 & 0.05 & 33.09 & 0.05 & 22.07 & 0.05 & 70.94 & 0.05 \\
& ResNet-50
& 88.70 & 1.00 & 38.48 & 1.00 & 23.04 & 1.00 & 80.59 & 1.00 \\
\midrule
\multirow{2}{*}{Router} & Benchmark Routing
& 87.48 & 0.76 & 36.49 & 0.58 & 23.63 & 0.40 & 79.04 & 0.85 \\
& EmbedLLM
& 87.52 & 0.63 & 36.25 & 0.56 & 22.41 & 0.46 & 79.86 & 0.87 \\
\midrule
\multirow{2}{*}{Cascade} & Post-Hoc-Embed
& 87.24 & 0.47 & 36.14 & 0.64 & 22.03 & 0.69 & 77.68 & 0.55 \\
& Margin Sampling
& 88.85 & 0.45 & 38.23 & 0.57 & 23.11 & 0.63 & 79.99 & 0.55 \\
\midrule
Ensemble & Asymmetric Duo
& \underline{89.96} & 1.05 & 38.89 & 1.05 & 25.16 & 1.05 & \underline{80.49} & 1.05 \\
\midrule
\multirow{2}{*}{Ours} & \textbf{CAUC}
& 89.81 & 0.43 & \textbf{39.17} & 0.56 & \textbf{26.08} & 0.64 & 80.35 & 0.55 \\
& \textbf{CAUC-RF}
& \textbf{90.03} & 0.43 & \underline{38.98} & 0.56 & \underline{25.64} & 0.64 & \textbf{80.70} & 0.55 \\
\bottomrule
\end{tabular}
\caption{Two-model image classification results for MnasNet-0.75 $\rightarrow$ ResNet-50. Accuracy and macro-F1 are percentages; relative GFLOPs are normalized to ResNet-50. Bold and underlining indicate the highest and second-highest predictive performance, respectively, for each dataset setting.}
\label{tab:image-main}
\end{table*}

\paragraph{Benchmarks.}
The LLM evaluation uses six multiple-choice benchmarks: MMLU \citep{hendrycks2021mmlu}, LogiQA \citep{liu2020logiqa}, MathQA \citep{amini2019mathqa}, MedMCQA \citep{pal2022medmcqa}, PIQA \citep{bisk2020piqa}, and SocialIQA \citep{sap2019socialiqa}. The vision evaluation uses Caltech256 \citep{griffin2007caltech256}, iWildCam from WILDS \citep{koh2021wilds}, and ImageNet \citep{russakovsky2015imagenet}; iWildCam includes in-distribution (ID) and out-of-distribution (OOD) test settings.

\paragraph{Baselines and inference setting.}
Baselines include individual small and large models as cost and performance endpoints. The routing baselines are RouterDC \citep{NIPS2024routerDC}, Benchmark Routing \citep{shnitzer2024routing}, RouteLLM \citep{ong2025routellm}, GraphRouter \citep{ICLR2025graph}, and EmbedLLM \citep{ICLR2025embedllm}. We compare FrugalGPT \citep{TMLR2024frugalgpt}, Margin Sampling \citep{ramirez2024optimising}, and Post-Hoc-Embed \citep{gupta2024cascades} as cascades, and Asymmetric Duo \citep{zhou2025asymmetric} as an ensemble. Tunable routers use performance-first settings, whereas cascade baselines use approximately cost-matched operating points. All policies use the same examples and saved model outputs.

For datasets with test labels, we divide the original validation set equally into $\Dcal$ and $\Dval$, reserving the test set for final evaluation. PIQA and MedMCQA lack test labels, so their original validation sets serve as test sets. For each, an equally sized pool from the training set is divided equally into $\Dcal$ and $\Dval$.

\paragraph{Evaluation metrics.}
LLM quality uses exact-match multiple-choice accuracy. Caltech256 and ImageNet use top-1 accuracy, whereas ID and OOD iWildCam use macro-F1 due to class imbalance. Language-model cost is nominal parameter count in billions, a uniform proxy when invocation prices are unavailable. Vision-model cost uses forward-pass GFLOPs from TorchVision metadata, normalized within each cascade so that the larger model costs $1.0\times$.

\subsection{Main Results on Language Models}

Across the six tasks in Table~\ref{tab:llm-main}, CAUC-RF always improves over large-only inference, while Base improves on five and is effectively tied on SocialIQA. The stronger CAUC variant gains between 0.09 and 2.49 percentage points. Base and CAUC-RF average 70.43\% and 70.44\% accuracy at a mean cost of 21.73, compared with 69.12\% at 26.00 for large-only inference. Asymmetric Duo reaches 70.84\% at cost 34.00, so CAUC approaches ensemble accuracy below the large-model endpoint cost.

Figure~\ref{fig:llm-acc-cost-all-points} shows full budget sweeps on two representative tasks. Base and CAUC-RF share a routing schedule, so their separation at identical costs isolates the prediction rule. Both occupy the strongest high-accuracy region, with gains largely saturating at moderate cost while competing cascades plateau at lower accuracy.

\subsection{Main Results on Image Models}

In Table~\ref{tab:image-main}, CAUC achieves the best predictive performance in all four settings, with CAUC-RF leading on Caltech256 and ImageNet and Base on both iWildCam settings. Relative to always-on Asymmetric Duo, CAUC improves every reported metric while reducing relative GFLOPs from 1.05 to 0.43--0.64. CAUC also outperforms ResNet-50 on Caltech256 and both iWildCam settings, while CAUC-RF improves ImageNet accuracy by 0.12 points. The two variants have identical per-dataset costs, isolating the effect of their prediction policies.

CAUC was most compute-efficient when the small model was already accurate. On Caltech256 and ImageNet, it stopped more examples after MnasNet-0.75 and surpassed both routers at 0.43 and 0.55 relative GFLOPs. On iWildCam, MnasNet-0.75 achieved only 33.09 and 22.07 macro-F1 in the ID and OOD settings, raising CAUC's costs to 0.56 and 0.64. CAUC remained cost-competitive on ID but exceeded both routers on OOD, where single-model routing avoids the cascade's double evaluation of deferred examples.

\begin{figure}[t]
\centering
\includegraphics[width=\columnwidth]{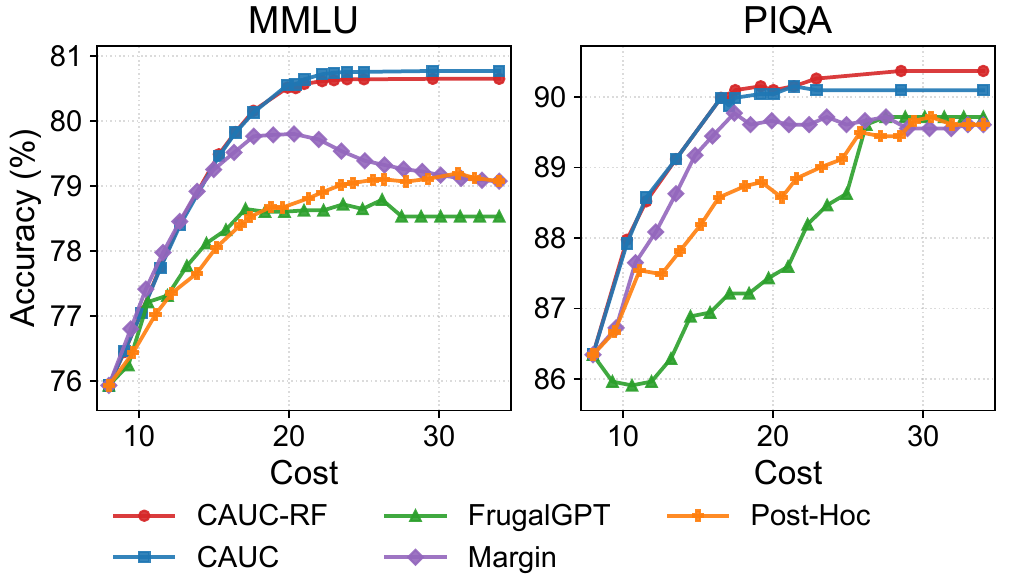}
\caption{Accuracy--cost trade-offs across operating points on MMLU and PIQA.}
\label{fig:llm-acc-cost-all-points}
\end{figure}

\subsection{Calibration and Threshold Reliability}

\begin{figure}[t]
\centering
\begin{minipage}[t]{0.495\linewidth}
\centering
\includegraphics[width=\linewidth,height=0.245\textheight,keepaspectratio]{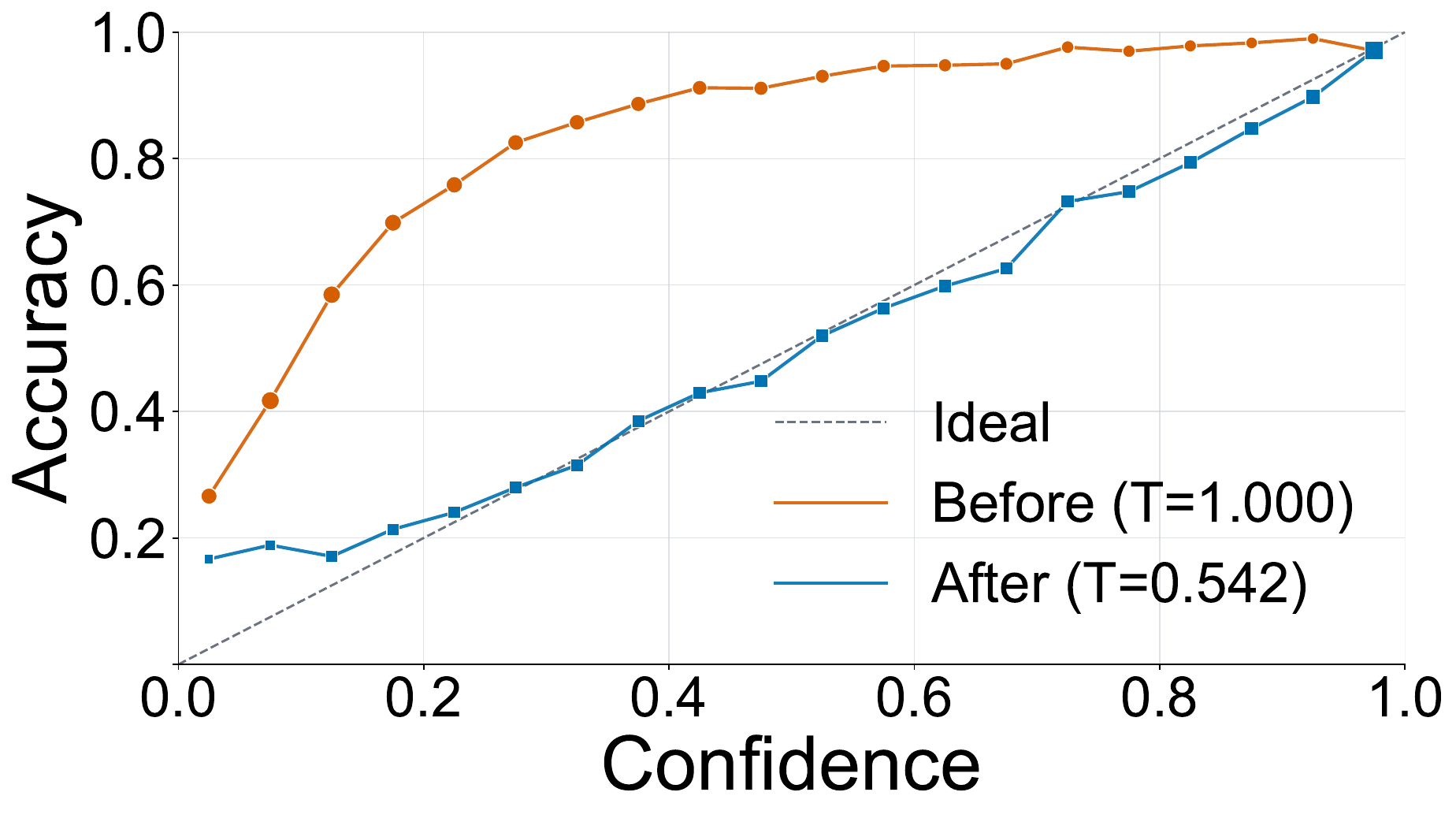}
\par\footnotesize (a) MnasNet-0.75 confidence before and after calibration.
\end{minipage}\hfill
\begin{minipage}[t]{0.495\linewidth}
\centering
\includegraphics[width=\linewidth,height=0.245\textheight,keepaspectratio]{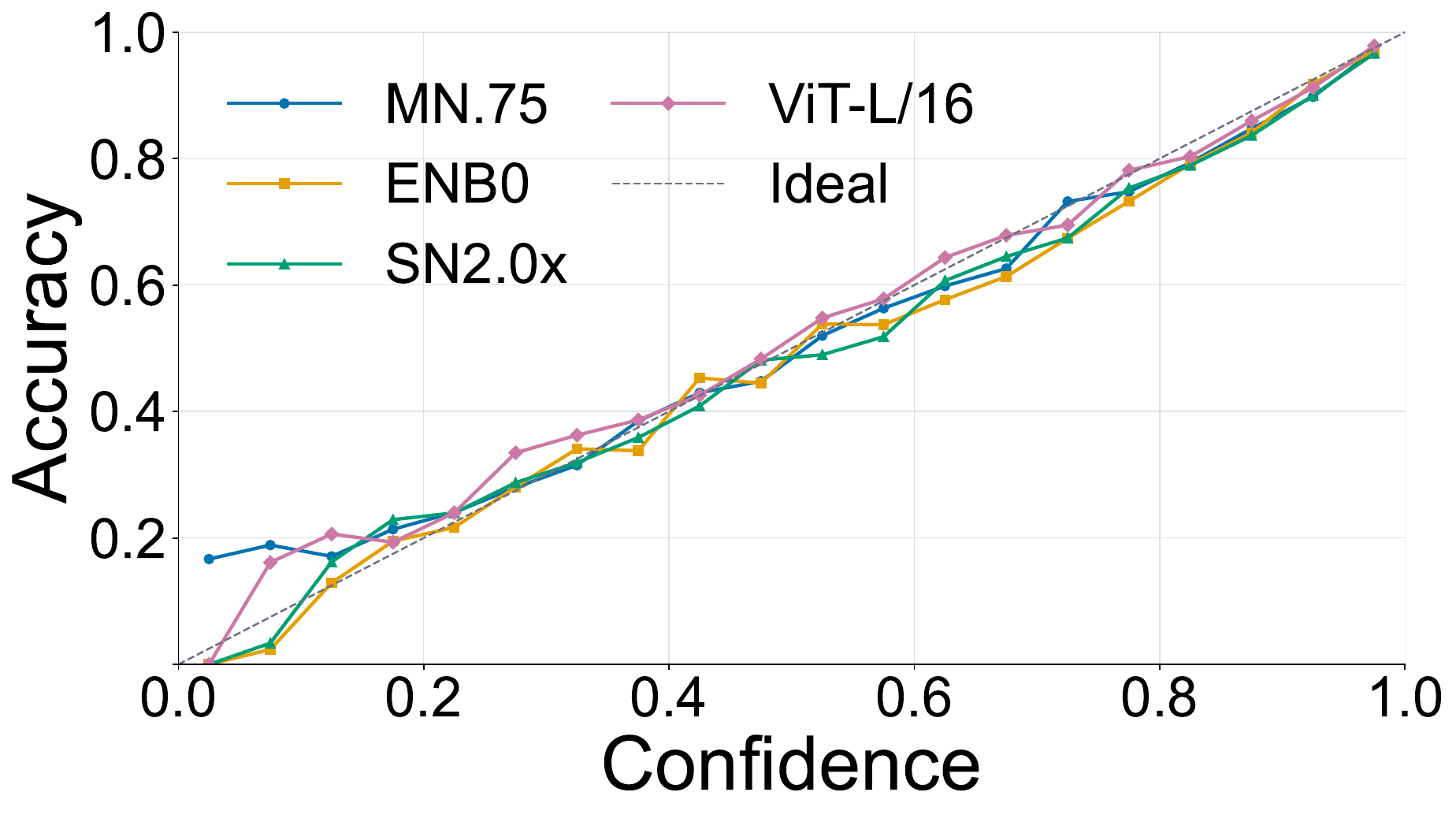}
\par\footnotesize (b) Accuracy of four models at matched calibrated confidence.
\end{minipage}

\vspace{0.6em}
\begin{minipage}[t]{0.495\linewidth}
\centering
\includegraphics[width=\linewidth]{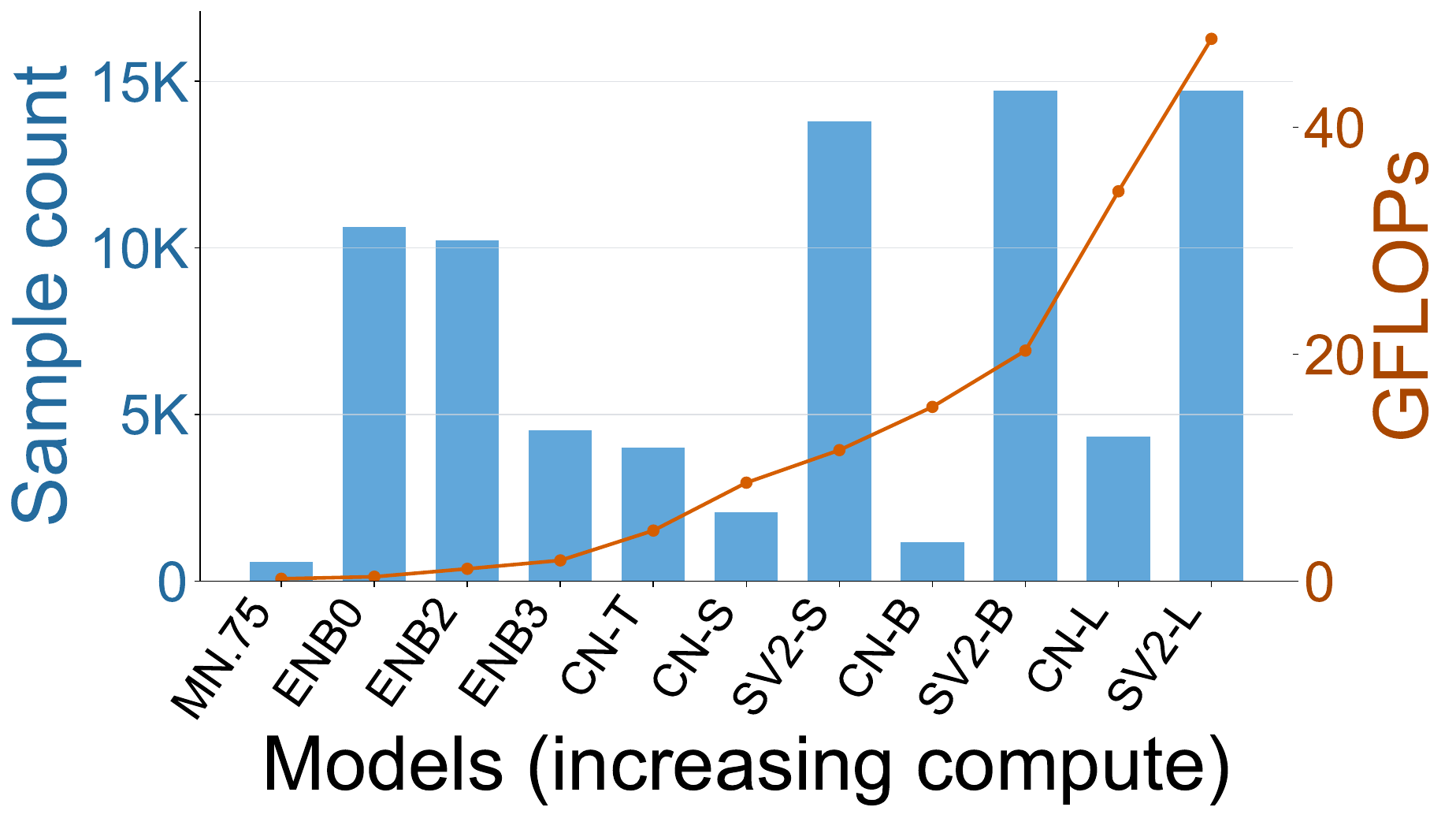}
\par\footnotesize (c) High-confidence sample counts before calibration.
\end{minipage}\hfill
\begin{minipage}[t]{0.495\linewidth}
\centering
\includegraphics[width=\linewidth]{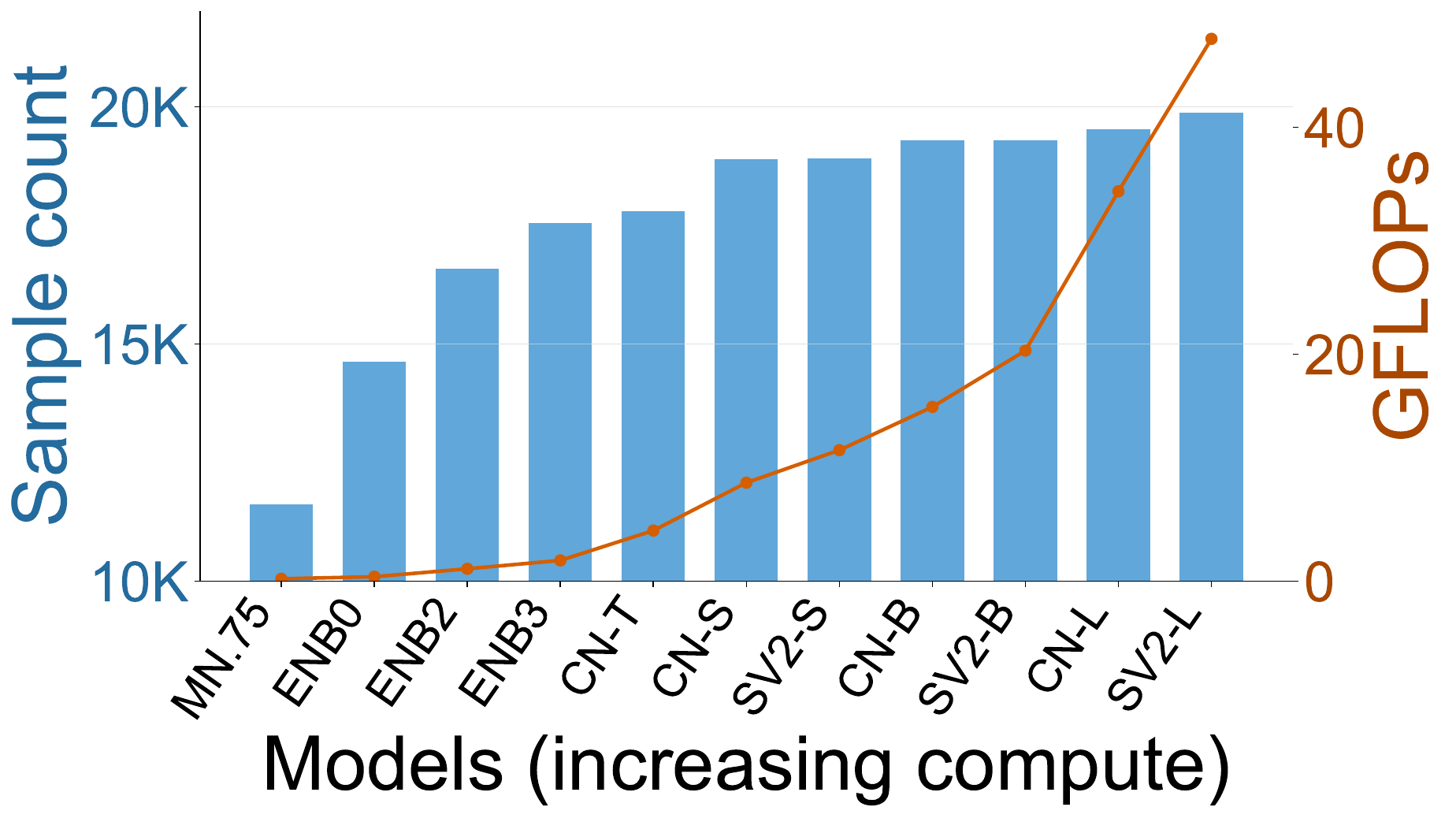}
\par\footnotesize (d) High-confidence sample counts after calibration.
\end{minipage}
\caption{Calibration reliability and high-confidence coverage on ImageNet. High-confidence counts use $[0.85,1.00]$, and the orange curve denotes model GFLOPs.}
\label{fig:calibration-effect}
\end{figure}

Figure~\ref{fig:calibration-effect}(a) shows that the uncalibrated MnasNet-0.75 curve lies above the diagonal across most confidence bins, indicating understated empirical correctness. Temperature scaling with $T=0.542$ moves the curve near the diagonal without changing the predicted class. Panel (b) shows that MnasNet-0.75, EfficientNet-B0, ShuffleNetV2-2.0x, and ViT-L/16 remain near a common confidence--accuracy relation after calibration. Calibrated scores therefore approximate correctness probabilities on a shared scale, supporting one reliability threshold across the cascade. Figures~\ref{fig:calibration-effect}(c) and (d) show that calibration makes high-confidence coverage increase more consistently with model compute. The shared cutoff therefore measures coverage at a common reliability target rather than model-specific score scaling.

\subsection{Complementarity-Guided Fusion}

\begin{table}[t]
\centering
\scriptsize
\setlength{\tabcolsep}{2.4pt}
\resizebox{\columnwidth}{!}{%
\begin{tabular}{@{}lrrrr@{}}
\toprule
\textbf{Model Pair} & \textbf{Direct} & \textbf{Fuse} & \textbf{CAUC-RF} & \textbf{CR (\%)} \\
\midrule
Ministral-3-3B $\rightarrow$ Ministral-3-8B
& \textbf{75.94} & 75.14 & 75.99 & $-1.70$ \\
Qwen3-VL-4B $\rightarrow$ Ministral-3-8B
& 75.94 & \textbf{76.48} & 76.73 & $+0.54$ \\
\midrule
Qwen3-4B $\rightarrow$ Qwen3-8B
& \textbf{74.93} & 74.66 & 75.09 & $-0.52$ \\
Ministral-3-3B $\rightarrow$ Qwen3-8B
& 74.93 & \textbf{75.50} & 75.64 & $+0.59$ \\
\midrule
Qwen3-4B $\rightarrow$ Ministral-3-14B
& \textbf{78.74} & 76.64 & 78.98 & $-1.70$ \\
Qwen3.5-4B $\rightarrow$ Ministral-3-14B
& 78.74 & \textbf{78.87} & 79.43 & $+0.20$ \\
\bottomrule
\end{tabular}%
}
\caption{
Confidence-weighted fusion for six LLM pairs on MMLU. Accuracy and CR are in percentage points. Bold marks the better of direct deferral and fusion; CAUC-RF is the cost-constrained reference.
}

\label{tab:fusion-admission}
\end{table}

\begin{figure}[t]
\centering
\includegraphics[width=0.9\columnwidth]{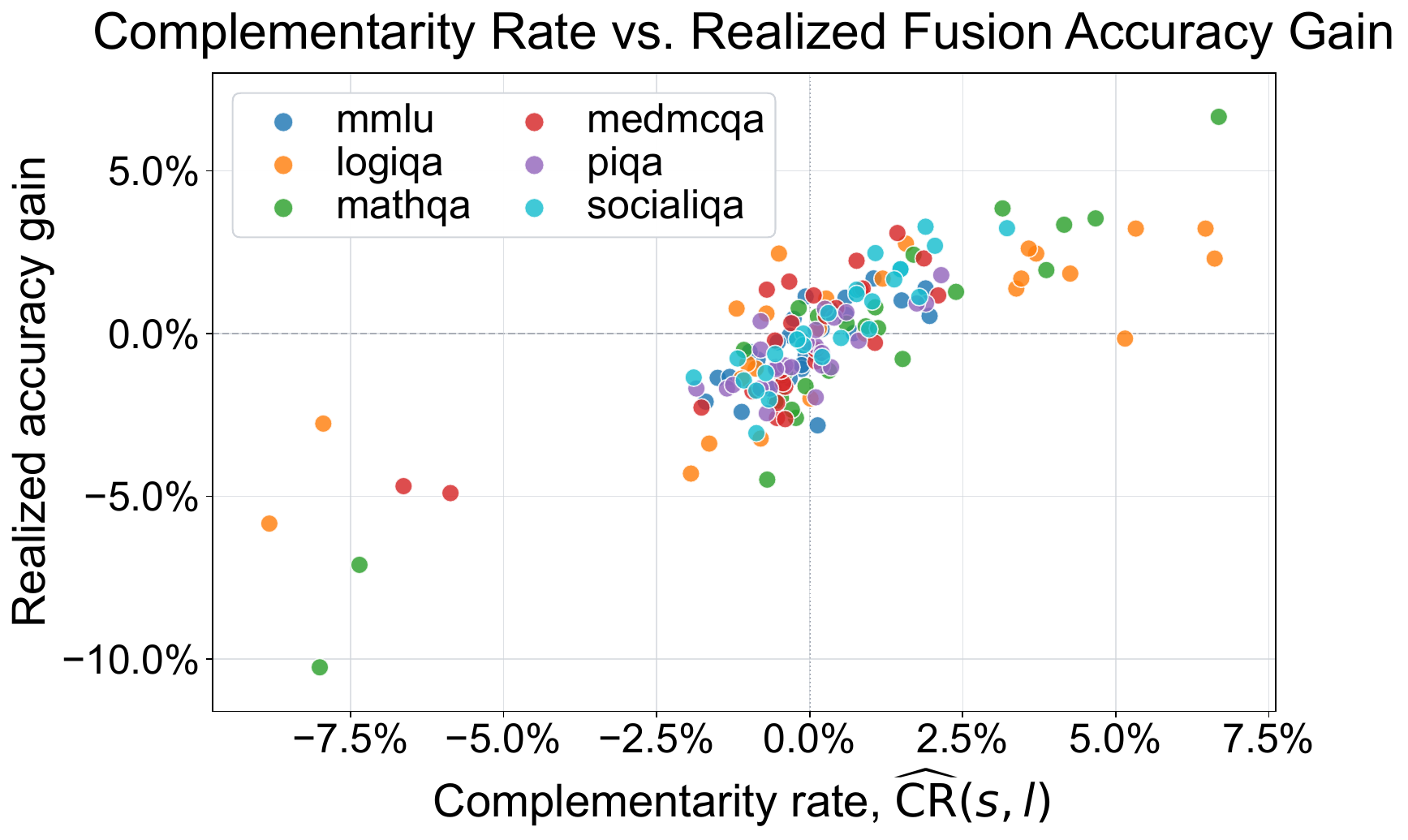}
\caption{Calibration-set complementarity rate versus realized accuracy gain from confidence-weighted fusion over direct deferral. Each point represents one dataset--model-pair evaluation across six LLM tasks; dashed lines indicate zero complementarity and zero gain.}
\label{fig:complementarity-gain}
\end{figure}

\begin{table}[t]
\centering
\scriptsize
\setlength{\tabcolsep}{2.4pt}
\renewcommand{\arraystretch}{0.80}
\resizebox{0.92\columnwidth}{!}{%
\begin{tabular}{@{}clrrrr@{}}
\toprule
\multirow{2}{*}{\textbf{\# Models}}& \multirow{2}{*}{\textbf{Newly Added Model}} & \multicolumn{2}{c}{\textbf{CAUC}} & \multicolumn{2}{c}{\textbf{CAUC-RF}} \\
\cmidrule(lr){3-4}\cmidrule(lr){5-6}
 & & Acc. & Cost & Acc. & Cost \\
\midrule
1 & Swin-V2-L (large only)
  & 86.74 & 47.80 & 86.74 & 47.80 \\
\midrule
2 & MnasNet-0.75
  & \textbf{86.22} & 28.93 & 86.26 & 28.93 \\
3 & EfficientNet-B0
  & 86.14 & 21.74 & 86.22 & 22.61 \\
4 & ResNet-50
  & 86.05 & \textbf{17.69} & 86.30 & \textbf{19.72} \\
5 & ConvNeXt-S
  & 85.96 & 18.83 & \textbf{86.48} & 21.75 \\
6 & Swin-V2-B
  & 85.73 & 19.55 & 86.34 & 24.12 \\
7 & ConvNeXt-B
  & 85.66 & 21.92 & 86.31 & 27.61 \\
\bottomrule
\end{tabular}%
}
\caption{
ImageNet results as the cascade grows from one to seven models, with Swin-V2-L fixed as the final endpoint. Each row adds the indicated model to the preceding cascade; the one-model row denotes large-only inference. 
}
\label{tab:multimodel-scale}
\end{table}

Using the same saved two-model outputs, we compare two CAUC variants, with either direct deferral or confidence-weighted fusion, and test whether the calibration-set complementarity rate predicts which is better. 
Across all six pairs in Table~\ref{tab:fusion-admission}, negative complementarity corresponds to higher direct-deferral accuracy, whereas positive complementarity corresponds to higher fusion accuracy. The zero-point rule therefore selects the better fixed fallback in every case, after which CAUC-RF improves accuracy by 0.05--0.56 points. Figure~\ref{fig:complementarity-gain} extends this comparison across model pairs and datasets. Although the signs predominantly align, some points cross the zero-gain boundary, supporting Eq.~\ref{eq:complementarity-rate} while showing that complementarity rate is an admission signal rather than a guarantee.

\subsection{Scaling to Longer Model Cascades}

We scale an ImageNet cascade from one to seven models, retaining Swin-V2-L as the endpoint, and measure accuracy and cost. 
Table~\ref{tab:multimodel-scale} shows that both policies minimize cost with four models, reducing the large-only cost by 63.0\% for CAUC and 58.8\% for CAUC-RF. CAUC-RF reaches its highest cascade accuracy of 86.48\% with five models, 0.26 points below large-only inference at less than half the cost. Further stages fail to improve accuracy and raise cost, indicating that marginal interception no longer offsets forward-pass cost; CAUC-RF remains more accurate than Base at a growing cost premium.
A fixed three-model LLM baseline comparison appears in the supplementary material.

\section{Conclusion}


CAUC uses independently calibrated confidence as a shared decision interface for heterogeneous model collaboration. It supports early acceptance, selective deferral, and confidence-weighted fusion without training a pool-specific router, while CAUC-RF recursively integrates predictions from earlier models in longer cascades. Our analysis links calibration error to accepted-sample reliability and bounds the regret of unified-threshold decisions. Across six LLM benchmarks, CAUC improves the accuracy--cost trade-off while avoiding about 47\% of strong-model calls. Across three image datasets, it maintains or improves predictive performance while reducing measured GFLOPs by up to 57\%. Fusion and multi-model experiments further demonstrate the benefits of model complementarity and the diminishing returns of adding excessive cascade stages.

\clearpage

\bibliography{uncertainty_cascade_draft}

\clearpage
\appendix
\setcounter{secnumdepth}{1}
\setcounter{equation}{0}
\setcounter{table}{0}
\setcounter{figure}{0}
\setcounter{proposition}{0}
\renewcommand{\theequation}{A\arabic{equation}}
\renewcommand{\thetable}{A\arabic{table}}
\renewcommand{\thefigure}{A\arabic{figure}}
\renewcommand{\theproposition}{A\arabic{proposition}}

\twocolumn[
\begin{center}
    {\LARGE \textbf{Appendix}}
\end{center}
\vspace{3em}
]

This appendix provides supplementary details for CAUC. Section~A presents
the complete algorithms for the Base and Recursive Fusion variants.
Section~B describes the models, dataset splits, baselines, and three-model
setting. Section~C reports additional language-model results for two- and
three-model cascades. Section~D compares method-specific setup times.
Finally, Section~E extends the unified-threshold analysis to
confidence-dependent fallback quality and proves the corresponding
performance-gap bound.

\section{Overall Algorithm Description}

\begin{algorithm}[h]
\footnotesize
\raggedright
\caption{Base Calibration-Aware Uncertainty Cascade}
\label{alg:cauc-base}
\begin{algorithmic}[1]
\REQUIRE Calibration set $\Dcal$; ordered models
$f_1,\ldots,f_M$; test input $x$
\ENSURE Prediction $\hat y(x)$
\FOR{$j=1,\ldots,M$}
    \STATE Fit $T_j>0$ by minimizing the NLL on $\Dcal$
    \STATE Set $\ell_j=z_j/T_j$, $\pi_j=\operatorname{softmax}(\ell_j)$,
    and $p_j=\max_k\pi_{j,k}$
    \STATE Estimate global $\mu_j,\sigma_j$ of $\ell_j$ on $\Dcal$
\ENDFOR
\STATE $\tau_0\leftarrow$ calibration accuracy of the final model $f_M$
\STATE $\mathcal S\leftarrow\{M\}$
\FOR{$j=1,\ldots,M-1$}
    \STATE Compute $\widehat{\mathrm{CR}}(j,M)
    =(N_{\mathrm{gain}}-N_{\mathrm{loss}})/|\Dcal|$
    \IF{$\widehat{\mathrm{CR}}(j,M)>0$}
        \STATE $\mathcal S\leftarrow\mathcal S\cup\{j\}$
    \ENDIF
\ENDFOR
\FOR{$j=1,\ldots,M-1$}
    \STATE Evaluate $f_j(x)$ and compute $\ell_j(x),\pi_j(x),p_j(x)$
    \IF{$p_j(x)\geq\tau_0$}
        \RETURN $\arg\max_k\ell_{j,k}(x)$
    \ENDIF
\ENDFOR
\STATE Evaluate $f_M(x)$ and compute $\ell_M(x),\pi_M(x),p_M(x)$
\STATE $q_j(x)\leftarrow
(\ell_j(x)-\mu_j)/(\sigma_j+\varepsilon)$ for every $j\in\mathcal S$
\STATE $z_{\mathrm{fuse}}(x)\leftarrow
\bigl(\sum_{j\in\mathcal S}p_j(x)q_j(x)\bigr)/
\bigl(\sum_{j\in\mathcal S}p_j(x)\bigr)$
\RETURN $\arg\max_k z_{\mathrm{fuse},k}(x)$
\end{algorithmic}
\end{algorithm}

\begin{algorithm}[h]
\footnotesize
\raggedright
\caption{CAUC with Recursive Fusion (RF)}
\label{alg:cauc-rf}
\begin{algorithmic}[1]
\REQUIRE Calibration set $\Dcal$; cascade-validation set $\Dval$;
ordered models $f_1,\ldots,f_M$; costs $c_1,\ldots,c_M$; budget $B$;
test input $x$
\ENSURE Prediction $\hat y(x)$
\FOR{$j=1,\ldots,M$}
    \STATE Fit $T_j>0$ by minimizing the NLL on $\Dcal$
    and set $\ell_j=z_j/T_j$
\ENDFOR
\STATE Set $r_1\leftarrow\ell_1$
\FOR{$j=2,\ldots,M$}
    \STATE Fit $\alpha_j,\beta_j>0$ on $\Dcal$ by minimizing the NLL of
    $\operatorname{softmax}\!\left(
    (r_{j-1}/\alpha_j+\ell_j/\beta_j)/2\right)$
    \STATE $r_j\leftarrow
    (r_{j-1}/\alpha_j+\ell_j/\beta_j)/2$
\ENDFOR
\STATE $\tau_B\leftarrow
\arg\max_{\tau:\,\widehat c_{\mathrm{val}}(\tau)\leq B}
\widehat A_{\mathrm{val}}(\tau)$ on $\Dval$
\STATE Evaluate $f_1(x)$ and set $r\leftarrow\ell_1(x)$
\FOR{$j=1,\ldots,M$}
    \IF{$j>1$}
        \STATE Evaluate $f_j(x)$ and set
        $r\leftarrow(r/\alpha_j+\ell_j(x)/\beta_j)/2$
    \ENDIF
    \STATE $p\leftarrow\max_k\operatorname{softmax}(r)_k$
    \IF{$p\geq\tau_B$ \OR $j=M$}
        \RETURN $\arg\max_k r_k$
    \ENDIF
\ENDFOR
\end{algorithmic}
\end{algorithm}

Algorithms~\ref{alg:cauc-base} and~\ref{alg:cauc-rf} summarize the two CAUC variants as complete offline--online procedures. Both variants first fit an independent temperature for every model on the calibration set, so that confidence has a comparable correctness interpretation across heterogeneous architectures, and then evaluate models from inexpensive to expensive. Base uses the final model's calibration accuracy as its early-exit threshold, retains an earlier model for deferred-example fusion only when its empirical complementarity rate is positive, and fuses the retained standardized logits using calibrated confidence weights. RF instead recalibrates and averages the accumulated prediction whenever a new model is invoked, and selects one shared stopping threshold on the independent cascade-validation set under budget \(B\). All temperatures, normalization statistics, complementarity decisions, fusion parameters, and stopping thresholds are fitted once before deployment; final evaluation labels are never used by either algorithm.

\begin{table*}[t]
\centering
\footnotesize
\setlength{\tabcolsep}{2.0pt}
\begin{tabular*}{\textwidth}{@{\extracolsep{\fill}}cl*{6}{rr}@{}}
\toprule
\multicolumn{2}{c}{\multirow{2}{*}{Method}}
& \multicolumn{2}{c}{MMLU}
& \multicolumn{2}{c}{LogiQA}
& \multicolumn{2}{c}{MathQA}
& \multicolumn{2}{c}{MedMCQA}
& \multicolumn{2}{c}{PIQA}
& \multicolumn{2}{c}{SocialIQA} \\
\cmidrule(lr){3-4}\cmidrule(lr){5-6}
\cmidrule(lr){7-8}\cmidrule(lr){9-10}\cmidrule(lr){11-12}
\cmidrule(lr){13-14}
\multicolumn{2}{c}{} & Acc. & Cost & Acc. & Cost & Acc. & Cost
& Acc. & Cost & Acc. & Cost & Acc. & Cost \\
\midrule
\multirow{3}{*}{Naive} & Qwen3.5-4B
& 72.34 & 4.00 & 58.53 & 4.00 & 36.18 & 4.00
& 58.10 & 4.00 & 85.69 & 4.00 & 76.35 & 4.00 \\
& Ministral-3-14B
& 78.74 & 14.00 & 56.53 & 14.00 & 50.08 & 14.00
& 66.48 & 14.00 & 85.85 & 14.00 & 75.67 & 14.00 \\
& Gemma-4-31B
& 82.22 & 31.00 & 59.75 & 31.00 & 50.75 & 31.00
& 69.46 & 31.00 & \textbf{92.55} & 31.00 & 75.85 & 31.00 \\
\midrule
\multirow{3}{*}{Router} & RouterDC
& 81.83 & 28.86 & 58.53 & 4.22 & 50.08 & 14.00
& 69.78 & 28.63 & \textbf{92.55} & 30.99 & 76.35 & 4.00 \\
& GraphRouter
& 72.57 & 4.79 & 57.76 & 8.35 & 36.18 & 4.00
& 58.88 & 5.81 & 85.91 & 4.78 & 76.21 & 4.09 \\
& EmbedLLM
& 79.04 & 19.10 & 59.91 & 17.77 & 50.62 & 16.39
& 67.93 & 19.49 & 90.21 & 23.57 & 75.99 & 19.68 \\
\midrule
Cascade & FrugalGPT
& 80.20 & 25.30 & 59.14 & 13.33 & 50.82 & 30.66
& 67.47 & 18.29 & 91.40 & 24.46 & 76.35 & 21.78 \\
\midrule
\multirow{2}{*}{CAUC (Ours)} & Base
& 82.95 & 24.65 & \textbf{60.68} & 25.14 & 51.91 & 32.02
& 70.17 & 27.02 & 91.78 & 23.61 & 78.42 & 18.09 \\
& RF
& \textbf{82.97} & 25.10 & 59.45 & 26.72 & \textbf{52.41} & 32.26
& \textbf{70.77} & 27.35 & 92.22 & 24.35 & \textbf{78.78} & 19.28 \\
\bottomrule
\end{tabular*}
\caption{Three-model results for Qwen3.5-4B $\rightarrow$
Ministral-3-14B $\rightarrow$ Gemma-4-31B. Accuracy is reported in
percentage points, and Cost follows the common evaluation accounting. Bold
marks the best accuracy in each task.}
\label{tab:multimodel-three}
\end{table*}

\section{Additional Experimental Details}

\paragraph{Models.}
The language experiments use models from the DeepSeek-R1-Distill, Gemma 4,
Ministral 3, Qwen3, Qwen3-VL, and Qwen3.5 families. The vision experiments
use models from the ConvNeXt, EfficientNet/EfficientNetV2, MnasNet, ResNet,
ShuffleNetV2, Swin-V2, and ViT families.

\paragraph{Dataset splits.}
For MMLU, LogiQA, MathQA, SocialIQA, Caltech256, and iWildCam, which provide
labeled test partitions in our evaluation setup, we divide the original
validation set equally into a calibration set and a cascade-validation set
for threshold selection, and reserve the official test set for final
evaluation. The official iWildCam ID and OOD test partitions are retained as
two separate final evaluation settings. PIQA and MedMCQA do not provide test
labels, so their original validation sets become the final evaluation sets;
for each dataset, we sample an equally sized pool from its training set and
divide that pool equally between calibration and threshold selection.
ImageNet likewise lacks public test labels, but its larger validation set
permits a direct three-way split: 25\% for calibration, 25\% for threshold
selection, and 50\% for final evaluation. Thus, no final evaluation example
is used to fit temperatures, choose fusion parameters, or select stopping
thresholds.

\paragraph{Baselines.}
\textbf{Small-only} evaluates the inexpensive model on every example and
serves as the minimum-cost endpoint, whereas \textbf{Large-only} always
evaluates the stronger model and serves as the single-model performance
endpoint.
\textbf{RouterDC} learns query and model embeddings with sample--model and
sample--sample contrastive objectives, then routes each query according to
their learned compatibility \citep{NIPS2024routerDC}.
\textbf{Benchmark Routing} reuses per-model correctness outcomes on existing
benchmarks to train binary performance predictors for selecting a model on a
new query \citep{shnitzer2024routing}.
\textbf{RouteLLM} learns a strong-versus-weak model routing decision from
human preference data so that easier queries can be sent to the cheaper
model \citep{ong2025routellm}.
\textbf{GraphRouter} represents tasks, queries, and LLMs in a heterogeneous
graph and predicts the performance and cost attributes of query--model edges
for routing \citep{ICLR2025graph}.
\textbf{EmbedLLM} first derives compact reusable representations of candidate
LLMs from their question--answer behavior and then trains a lightweight
downstream router on those representations \citep{ICLR2025embedllm}.
\textbf{FrugalGPT} learns an adaptive ordering and stopping policy that calls
LLM services sequentially while satisfying a user-specified budget
\citep{TMLR2024frugalgpt}.
\textbf{Margin Sampling} uses the small model's margin between its two
highest class scores as a nonparametric confidence measure and defers
low-margin examples to the large model \citep{ramirez2024optimising}.
\textbf{Post-Hoc-Embed} trains a post-hoc deferral predictor from token-level
uncertainty and representation features to decide when the stronger model
should be invoked \citep{gupta2024cascades}.
\textbf{Asymmetric Duo} evaluates a large model together with a smaller
sidekick and combines their predictions by learned weighted averaging
\citep{zhou2025asymmetric}; unlike a router or cascade, it therefore incurs
both model costs on every example.

\section{Additional Language-Model Results}

Table~\ref{tab:multimodel-three} evaluates the three-model Qwen3.5-4B
$\rightarrow$ Ministral-3-14B $\rightarrow$ Gemma-4-31B chain, whereas
Figure~\ref{fig:llm-acc-cost-remaining-tasks} completes the budget sweeps for
the main paper's two-model Ministral-3-8B $\rightarrow$ Gemma-4-26B system on
the other four tasks.

For the three-model setting, Table~\ref{tab:multimodel-three} shows that the
stronger CAUC variant
outperforms the largest single model on five of six tasks, with gains from
0.75 points on MMLU to 2.92 points on SocialIQA; on PIQA it is 0.33 points
lower. CAUC attains the best result on every task except PIQA. Averaged
across tasks, Base and RF reach 72.65\% and 72.77\% accuracy at costs
25.09 and 25.84, respectively, compared with 71.76\% at cost 31.00 for
Gemma-4-31B.

\begin{figure}[H]
\centering
\includegraphics[width=\columnwidth]{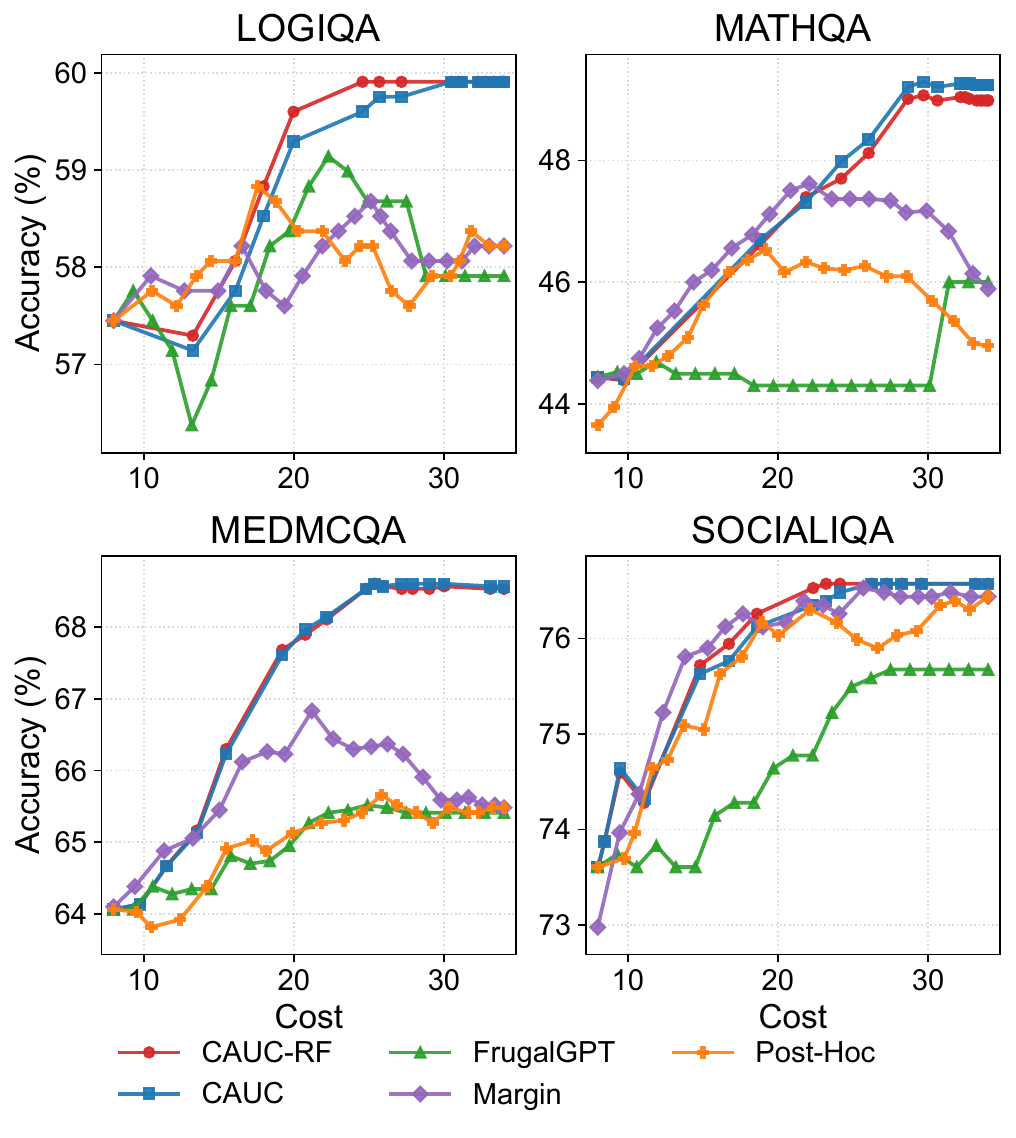}
\caption{Two-model accuracy--cost trade-offs for Ministral-3-8B
$\rightarrow$ Gemma-4-26B across operating points on LogiQA, MathQA,
MedMCQA, and SocialIQA, complementing the MMLU and PIQA results reported in
the main paper.}
\label{fig:llm-acc-cost-remaining-tasks}
\end{figure}

For the two-model setting, Figure~\ref{fig:llm-acc-cost-remaining-tasks}
shows Base and RF in the upper accuracy region at moderate-to-high costs on
all four tasks. Their advantage is clearest on MathQA and MedMCQA, while both
variants remain stable near 60\% accuracy on LogiQA. SocialIQA is the closest
comparison: Margin Sampling is competitive at intermediate costs, but CAUC
retains the highest endpoint. Thus, the selected results reflect the broader
accuracy--cost frontier.

\section{Method-Specific Setup Time}

\begin{table}[H]
\centering
\small
\setlength{\tabcolsep}{5pt}
\begin{tabular}{@{}llr@{}}
\toprule
Category & Method & Setup Time (s) $\downarrow$ \\
\midrule
Router & RouterDC & 4,891 \\
& Benchmark Routing & 156 \\
& RouteLLM & 161 \\
& GraphRouter & 449 \\
& EmbedLLM & 841 \\
\midrule
Cascade & FrugalGPT & 3,777 \\
& Margin Sampling & 16 \\
& Post-Hoc-Embed & 209 \\
\midrule
CAUC (Ours) & Base & \textbf{12} \\
& RF & 20 \\
\bottomrule
\end{tabular}
\caption{Offline runtime of method-specific policy construction. Time
includes router training or cascade calibration, fitting, and threshold
selection as required by each method. Base-model forward passes and
deployment inference are excluded. Bold marks the lowest setup time.}
\label{tab:runtime-comparison}
\end{table}

Table~\ref{tab:runtime-comparison} isolates the overhead of constructing the
routing or cascade policy from the cost of running the underlying models.
Base has the lowest setup time at 12 seconds, while RF requires 20
seconds; the additional time comes from selecting its budget-dependent
operating policy. Margin Sampling is the closest baseline at 16 seconds.
The learned routers require 156--4,891 seconds, making Base and RF
respectively 13.0--407.6$\times$ and 7.8--244.6$\times$ faster than this
group. The cascade baselines span 16--3,777 seconds: lightweight score
thresholding is inexpensive, whereas methods with broader search or fitting
procedures incur substantially larger offline cost. CAUC therefore adds
little method-specific preparation beyond calibration and, for RF,
development-set policy selection.

\section{Generalized Guarantee for Unified Thresholds}
\label{sec:supp-threshold-guarantee}

The second theoretical proposition in the main paper assumes that the fallback loss is constant. Here we define its two risks formally and extend the result to confidence-dependent fallback quality. The main-paper bound follows as a special case.

\subsection{Setup and Risk Definitions}

Let $P_s\in[0,1]$ be the small model's calibrated confidence, let $C_s$ indicate whether its prediction is correct, and define
\begin{equation}
    \eta_s(p)=\Pr(C_s=1\mid P_s=p),
    \qquad \ell_s(p)=1-\eta_s(p).
    \label{eq:supp-small-loss}
\end{equation}
The fallback is the fixed policy applied after deferral, such as the large-model prediction or a predetermined fusion rule. If its conditional error is $r_{\mathrm{fb}}(p)$, its cost-sensitive loss is
\begin{equation}
    L_{\mathrm{fb}}(p)=r_{\mathrm{fb}}(p)+\lambda c_l,
    \label{eq:supp-fallback-loss}
\end{equation}
where $c_l$ is the incremental large-model cost and $\lambda\geq0$ is its loss weight.

Fix a reference loss $\bar L\in[0,1]$ and let $\bar\tau=1-\bar L$. The unified-threshold policy accepts the small model when $P_s\geq\bar\tau$. Its risk is
\begin{equation}
\begin{aligned}
    \mathcal R_{\mathrm{threshold}}
    =\mathbb E\!\left[
      \ell_s(P_s)\mathbf 1\{P_s\geq\bar\tau\}
      +L_{\mathrm{fb}}(P_s)\mathbf 1\{P_s<\bar\tau\}
    \right].
    \label{eq:supp-threshold-risk}
\end{aligned}
\end{equation}
We compare it with a score-aware oracle that knows the two conditional loss functions but, like the threshold policy, may use only $P_s$ to choose between acceptance and fallback. It does not observe the true label of an individual example. Its risk is
\begin{equation}
    \mathcal R_{\mathrm{oracle}}
    =\mathbb E\!\left[
      \min\{\ell_s(P_s),L_{\mathrm{fb}}(P_s)\}
    \right].
    \label{eq:supp-oracle-risk}
\end{equation}

\subsection{Performance Guarantee}

\begin{proposition}[Unified-threshold performance gap]
\label{prop:supp-threshold-gap}
Let
\begin{equation}
\begin{aligned}
    \epsilon_s&=\mathbb E|\eta_s(P_s)-P_s|,\\
    \Delta_{\mathrm{fb}}(\bar L)
    &=\mathbb E|L_{\mathrm{fb}}(P_s)-\bar L|.
\end{aligned}
\end{equation}
Then
\begin{equation}
    0\leq
    \mathcal R_{\mathrm{threshold}}-
    \mathcal R_{\mathrm{oracle}}
    \leq\epsilon_s+\Delta_{\mathrm{fb}}(\bar L).
    \label{eq:supp-general-threshold-regret}
\end{equation}
\end{proposition}

\noindent\textit{Proof.}
For a fixed confidence $p$, define the true and surrogate differences between the acceptance and fallback losses as
\begin{equation}
\begin{aligned}
    d(p)&=1-\eta_s(p)-L_{\mathrm{fb}}(p),\\
    \widehat d(p)&=1-p-\bar L.
\end{aligned}
\label{eq:supp-loss-differences}
\end{equation}
The oracle accepts when $d(p)\leq0$, whereas the threshold policy accepts when $\widehat d(p)\leq0$, equivalently when $p\geq\bar\tau$. Let $\rho(p)$ denote the threshold policy's conditional excess loss. If the two policies agree, then $\rho(p)=0$. If they disagree, then $d(p)\widehat d(p)\leq0$ and $\rho(p)=|d(p)|$, which implies
\begin{equation}
\begin{aligned}
    \rho(p)
    &\leq|d(p)-\widehat d(p)|\\
    &\leq|\eta_s(p)-p|
      +|L_{\mathrm{fb}}(p)-\bar L|.
    \label{eq:supp-pointwise-regret}
\end{aligned}
\end{equation}
Taking expectation over $P_s$ proves the upper bound. Nonnegativity follows because the oracle minimizes conditional loss at every confidence value. \hfill$\square$

\subsection{Relation to the Main-Paper Result}

The main paper considers
\begin{equation}
    L_{\mathrm{fb}}(p)=\bar L
    =L_{\mathrm{large}}=r_l+\lambda c_l.
\end{equation}
Then $\bar\tau=1-L_{\mathrm{large}}=\tau^\star$ and $\Delta_{\mathrm{fb}}(\bar L)=0$. Equations~\ref{eq:supp-threshold-risk} and~\ref{eq:supp-oracle-risk} therefore give the precise meanings of $\mathcal R_{\mathrm{threshold}}$ and $\mathcal R_{\mathrm{oracle}}$ in the main paper, while Proposition~\ref{prop:supp-threshold-gap} reduces to
\begin{equation}
    \mathcal R_{\mathrm{threshold}}
    -\mathcal R_{\mathrm{oracle}}
    \leq\epsilon_s.
    \label{eq:supp-main-bound}
\end{equation}
Under perfect calibration, the threshold rule is consequently optimal among stopping policies based only on $P_s$.

For direct large-model fallback with $\lambda=0$, $\tau^\star=1-r_l=A_l$, and CAUC substitutes the calibration-set estimate $\widehat A_l$. If the fallback instead uses fusion, $\Delta_{\mathrm{fb}}(\bar L)$ captures variation and mismatch in its conditional quality. Thus, a unified threshold is near-optimal when both calibration error and fallback variation are small. This result concerns only the stopping decision; it does not establish optimality of the fusion rule or the recursive CAUC-RF extension, nor does it include finite-sample estimation error.

\end{document}